\documentclass{article}
\usepackage{microtype}
\usepackage{graphicx}
\usepackage{booktabs} 
\usepackage{xcolor}
\usepackage{amsmath,amssymb,amsthm,amsfonts}
\usepackage{dsfont}
\usepackage{subcaption}

\usepackage{algorithm}
\usepackage{algorithmic}

\definecolor{DarkGreen}{rgb}{0.1,0.5,0.1}

\def\R{\mathbb{R}}
\def\err{\textnormal{err}}
\def\errh{\widehat{\textnormal{err}}_\delta}
\def\elld{\ell_{\delta}}
\newcommand{\indic}[1]{\mathds{1} \Big( #1 \Big)}
\def\alg{\mathcal{A}}
\def\E{\mathop{\mathbb{E}}}
\def\Snot{\Breve{S}}

\usepackage{hyperref}

\hypersetup{
  colorlinks   = true,
  urlcolor     = blue,
  linkcolor    = blue,
  citecolor   = blue
}

\newif\ifnotes\notestrue

\ifnotes
\usepackage{color}
\definecolor{mygrey}{gray}{0.50}
\newcommand{\notename}[2]{{\textcolor{blue}{\footnotesize{\bf (#1:} {#2}{\bf ) }}}}

\else

\newcommand{\notename}[2]{{}}

\fi

\newcommand{\pnote}[1]{{\notename{Pranjal}{#1}}}

\usepackage{fullpage}
\usepackage[utf8]{inputenc} 
\usepackage[T1]{fontenc}    
\usepackage{hyperref}       
\usepackage{url}            
\usepackage{booktabs}       
\usepackage{amsfonts}       
\usepackage{amsthm}
\usepackage{natbib}

\newtheorem{theorem}{Theorem}
\newtheorem{definition}{Definition}
\newtheorem{lemma}{Lemma}

\title{A Multiclass Boosting Framework for Achieving Fast and Provable Adversarial Robustness}

\author{%
  Jacob Abernethy\\
  Georgia Institute of Technology\\
  \texttt{profi@gatech.edu} \\
  \and
  Pranjal Awasthi \\
  Google Research \\
  \texttt{pranjalawasthi@google.com} \\
  \and
  Satyen Kale \\
  Google Research \\
  \texttt{satyenkale@google.com}
}

\begin{document}
\date{}

\maketitle

\begin{abstract}
Alongside the well-publicized accomplishments of deep neural networks there has emerged an apparent bug in their success on tasks such as object recognition: with deep models trained using vanilla methods, input images can be slightly corrupted in order to modify output predictions, even when these corruptions are practically invisible. This apparent lack of robustness has led researchers to propose methods that can help to prevent an adversary from having such capabilities. The state-of-the-art approaches have incorporated the robustness requirement into the loss function, and the training process involves taking stochastic gradient descent steps not using original inputs but on adversarially-corrupted ones. In this paper we propose a multiclass boosting framework to ensure adversarial robustness. Boosting algorithms are generally well-suited for adversarial scenarios, as they were classically designed to satisfy a minimax guarantee. We provide a theoretical foundation for this methodology and describe conditions under which robustness can be achieved given a weak training oracle. We show empirically that adversarially-robust multiclass boosting not only outperforms the state-of-the-art methods, it does so at a fraction of the training time.
\end{abstract}

\section{Introduction}
\label{sec:intro}

The phenomenon of adversarial robustness corresponds to a classifier's susceptibility to small and often imperceptible perturbations made to the input at {\em test time}. In the context of deep neural networks, such vulnerabilities were first reported in the work of \citet{biggio2013evasion, szegedy2013intriguing}. Since then it has been empirically demonstrated that across a wide range of settings, vanilla-trained neural networks 
are susceptible to test time perturbations \citep{ebrahimi2017hotflip,carlini2018audio}. This has led to a flurry of research on proposed defenses to adversarial perturbations \citep{madry2017towards, wong2018provable, raghunathan2018certified, sinha2018certifying} and corresponding attacks \citep{carlini2017towards, sharma2017breaking} that aim to break them. See the supplementary material for a more exhaustive discussion of existing relevant literature on adversarial robustness.

One of the most popular methods for making neural networks robust to adversarial attacks is the projected gradient descent~(PGD) based {\em adversarial training} procedure of \citet{madry2017towards}. This procedure comprises of an alternate minimization approach where at each epoch, a given batch of examples is replaced by its adversarial counterpart. This is obtained by approximately solving the problem of finding the worst perturbation~(within a specified radius) for each example in the batch. The parameters of the network are then updated via stochastic gradient descent, but using the adversarial batch of examples. It has also been empirically well established that the complexity of the classifier plays a crucial role in ensuring robustness. For many datasets, the architectures used for training robust classifiers tend to be more complex than the ones used for standard training\footnote{For instance, on the CIFAR-10 dataset adversarial training is typically performed using variants of the ResNet architecture.}. As a result of using more complex architectures and solving a difficult maximization problem for each example in the batch, adversarial training takes a long time to converge, and is hard to scale to large datasets. Recent works have started to address the problem of designing faster methods for adversarial training \citep{shafahi2019adversarial, wong2020fast}. In a similar vein, the goal of our work is to design theoretically sound methods for fast adversarial training using a multiclass boosting approach. 
Boosting \citep{freund1996game, schapire2003boosting} is a theoretically grounded paradigm for the design of machine learning algorithms, and has enjoyed tremendous empirical success over the years. Boosting belongs to a general class of ensemble methods, and works by combining several
base learners, each trained to achieve only modest performance individually. An appropriately weighted combination of such {\em weak learners} leads to a classifier of arbitrarily high accuracy \citep{freund1996game}. 

What is appealing about the boosting methodology, from the perspective of adversarial training, is that the necessary guarantee for the repeated selection of the base learner is indeed extremely weak: in the binary classification setting one need only find a hypothesis that has (reweighted) prediction accuracy better than $\frac 1 2 + \gamma$---this is only slightly better than random guessing. It can be dramatically easier to find such a weakly-accurate predictors, and combine them into a high-accuracy predictor. In the context of training neural networks robust to adversarial examples, this suggests we can use small network architectures, train them for a limited amount of time to ensure they are weakly adversarially robust, and eventually ensemble these networks together. While a given trained base learner may have vulnerabilities with respect to particular inputs, the boosting approach is to then reweight the dataset to more strongly emphasize hard examples. The next base learner will focus more heavily on the examples where our predictors were susceptible during previous boosting iterations. 


It is worth noting that previous work has suggested that blindly ensembling models does not lead to better robustness, and often adversarial examples translate across models trained independently, either via different methods or using different architectures \citep{he2017adversarial}. Nevertheless, we rigorously show that using the right formulation adversarial boosting is indeed possible. Investigating the theory of boosting in adversarial settings was also recently posed an open problem in the work of \citet{montasser2020reducing}.

\textbf{Theoretical foundations of multiclass boosting for achieving adversarial robustness.} We identify a natural notion of a {\em robust weak learner} that is appropriate in adversarial settings. Via a reduction to standard multiclass boosting, we then show that given access to a robust weak learning oracle, boosting to arbitrarily high robust accuracy can be achieved in adversarial settings. We then extend our theoretical framework to incorporate score-based predictors. This leads to a natural greedy algorithm based on the {\em gradient boosting} framework \citep{friedman2001greedy} that facilitates an efficient implementation.

\textbf{State-of-art results in a fraction of previous training time.} Our general boosting based algorithm can be combined with any existing methods for training adversarially robust classifiers. We use our algorithm for training classifiers robust to $\ell_\infty$ perturbations and compare it with the PGD based algorithm of \citet{madry2017towards}. Furthermore, we also apply our algorithm for training certifiably robust models under $\ell_2$ perturbation and compare it to the recent work of \citet{salman2019provably}, achieving state-of-the-art results in a fraction of the training time for for previous adversarially robust models, as shown in Figure~\ref{fig:train-time}.
\begin{figure*}[ht]
  \hspace*{\fill}%
  \includegraphics[width=3in]{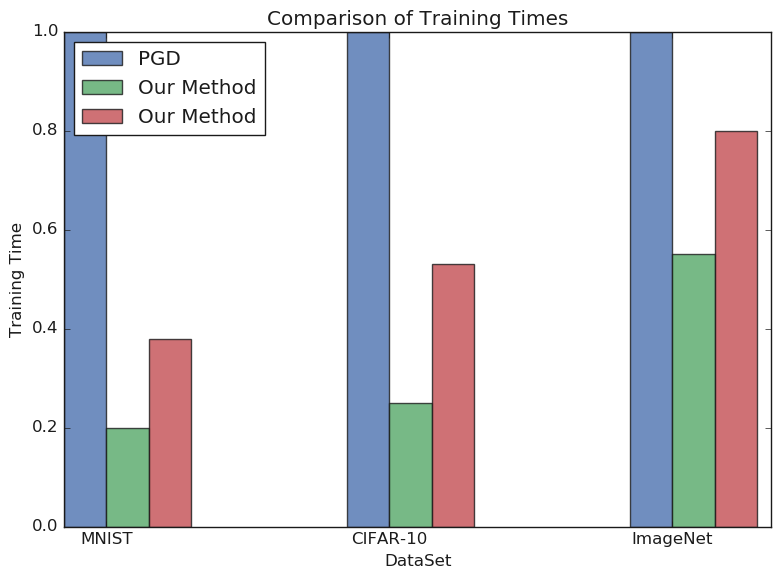}%
  \includegraphics[width=3in]{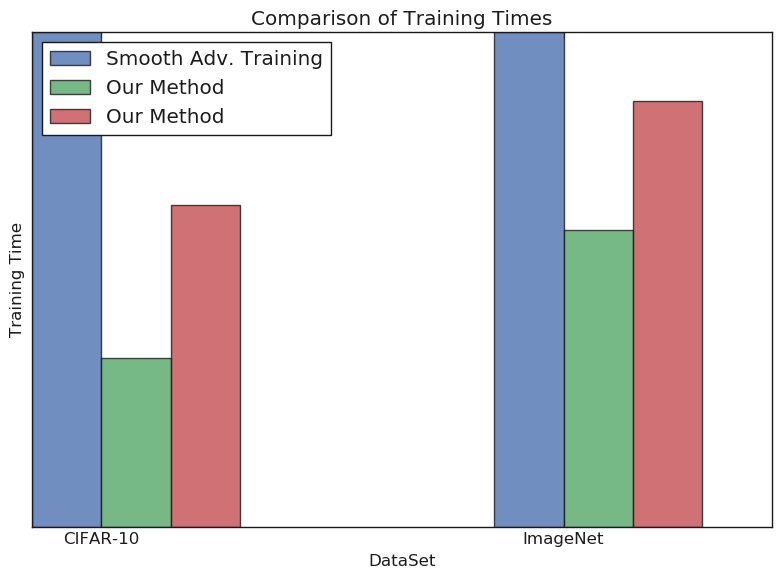}%
  \hspace*{\fill}%
  \caption{\label{fig:train-time}{\em Left:} Comparison of training times on MNIST, CIFAR-10 and ImageNet for $\ell_\infty$ robustness. The figure shows the training time for our boosting based approach for two different choice of the base weak learners, as a fraction of the time needed to train via the PGD based approach of \cite{madry2017towards}.
  {\em Right:} Comparison of training times on CIFAR-10 and ImageNet for certified $\ell_2$ robustness. The figure shows the training time for our boosting based approach for two different choice of the base weak learners, as a fraction of the time needed to train via the randomized smoothing approach of \cite{salman2019provably}}
\end{figure*}

\section{Related Work}
\label{sec:related}
There is a large body of work on algorithms for training adversarially robust models, certifying robustness and attacks for generating adversarial examples. Here we discuss existing literature that is most relevant to the current work. The current state of the art method that has resisted various adversarial attacks is the projected gradient descent~(PGD) based algorithm of \citep{madry2017towards}. It is well known that PGD based training takes significantly more time as compared to standard stochastic gradient descent based training. This is due to the fact that the training procedure approximately solves a difficult maximization problem, per example, per batch in each epoch. Recently there have been efforts to speed up PGD based training and scale the method to large datasets \citep{gowal2018effectiveness, shafahi2019adversarial, wong2020fast}. There have also been efforts to improve the robustness of models trained via PGD for $\ell_\infty$ perturbations, to other types of perturbations to the test input \citep{schott2018towards}. There have also been methods beyond the vanilla PGD based algorithm that have been proposed to train robust models \citep{gowal2018effectiveness, gowal2019alternative, salman2019provably, zhang2019theoretically, li2019certified}.

A series of recent works focus on designing attacks to evaluate and benchmark the robustness of machine learning models. These range from white-box attacks, black-box attacks to even physical attacks in the real world \citep{gilmer2018motivating}. The work of \citep{carlini2017towards} developed state of the art attacks for neural networks. The recent work of \citep{athalye2018obfuscated} showed that many defenses proposed for adversarial robustness suffer from the gradient masking phenomenon and can be broken by carefully designing the attack. Since evaluating models based on a particular attack is sensitive to the choice of the particular algorithm used for generating adversarial examples, there have also been recent works on provably certifying the robustness of models. These include using linear programming, semi-definite based relaxations \cite{wong2018provable, raghunathan2018certified} and Gaussian smoothing based methods for certifying $\ell_2$ robustness \cite{cohen2019certified}. In particular, the work of \citet{salman2019provably} uses the Gaussian smoothing based method of \citet{cohen2019certified, lecuyer2019certified} to propose optimizing a smooth adversarial loss for achieving certified $\ell_2$ robustness. In Section~\ref{sec:experiments} we use this smooth loss as a weak learner in our boosting based algorithm for achieving certified $\ell_2$ robustness. In the context of standard training, ensembling of neural networks has recently been shown to achieve superior performance \citep{loshchilov2016sgdr, smith2017cyclical, izmailov2018averaging, huang2017snapshot}. A crucial component in these works is the use of {\em cyclic learning rates}. Training proceeds by choosing a large learning rate and training the model for a few epochs over which the learning rate is reduced to zero. At the end of the epoch the checkpoint is saved and training is restarted with a large learning rate to train another network, and so on. This has the effect of forcing the network to explore different parts of the parameter space. As we discuss in Section~\ref{sec:experiments}, the use of cyclic learning rates will play a crucial role in our boosting based approach as well. In the context of adversarial robustness ensembling has been explore to a limited extent. The work of \citet{tramer2017ensemble} proposed performing training data augmentation by using adversarial examples generated on a different model. The recent work of \citet{andriushchenko2019provably} studies boosting for adversarial robustness in the context of decision stumps.

\section{Boosting framework}

We consider a multiclass classification task, where the input space is $\R^d$, and the output space has $k$ classes, and is denoted by $[k] := \{1, 2, \ldots, k\}$. Assume we have sampled a training set $S=\{(x_1, y_1), \dots, (x_m, y_m)\} \in (\R^d \times [k])^m$, from the data distribution. For a set $U$, we define $\Delta(U)$ to be the set of all probability distributions on $U$ for an appropriate $\sigma$-field that is evident from context. For any $p \geq 1$, define $B_p(\delta) := \{z:\R^d: \|z\|_p \leq \delta\}$.

A {\em multilabel} predictor is a function $h: \R^d \to \{0, 1\}^k$. The output $h(x)$ of a multilabel predictor is to be interpreted as the indicator vector of a set of potential labels for $x$, so we also use the notation $y \in h(x)$ to mean $h(x)_y = 1$. We call $h$ a {\em unilabel} predictor if the output $h(x)$ is a singleton for every input $x$. For convenience, for a unilabel predictor we use the notation $h(x) = y$ to mean $h(x) = \{y\}$. 

In practice, multiclass predictors are usually constructed via {\em score-based} predictors: this is a function $f: \R^d \to \R^k$, such that for any input $x$, the vector $f(x)$ is understood as assigning scores to each label, with higher scores indicating a greater degree of confidence in that label. A score based predictor $f$ can be converted into a {\em unilabel} predictor $f^\text{am}$ via the argmax operation, i.e. $f^\text{am}(x) = \arg\max_y f(x)_y$, with ties broken arbitrarily. 

\subsection{Multiclass Boosting}
\label{sec:multiclass-boosting}


Suppose $H$ is a class of multilabel predictors. To define appropriate conditions for boostability, we need to define a particular loss function. Let $\ell : H \times \R^d \times [k] \times [k]$ be
\[
    \textstyle \ell(h; x, y, y') = \indic{y' \in h(x)} - \indic{y \in h(x)}.
\]
This is an unusual loss function for a couple of reasons. First, it takes values in $\{-1, 0, 1\}$. Second, it depends on not one but two labels $y \ne y'$. Here we should think of $y$ as the ``correct label,'' and $y'$ as a ``candidate incorrect label,'' and the loss is measuring the extent to which $h$ predicted $y$ from $x$ \emph{relative to}  the wrong label $y'$. 

As it is core to the boosting framework, we will imagine that some distribution $Q \in \Delta(H)$ over the space of hypotheses is given to us, and we want to compute the expected error of the random $h$ drawn according to $Q$. In this case, we define $\ell(Q;x,y,y') = \E_{h \sim Q}[\ell(h; x,y,y')]$. The distribution $Q \in \Delta(H)$ also defines a {\em score-based} predictor $h_Q := \E_{h \sim Q}[h]$, and thereby, a {\em unilabel} predictor, which we call the \emph{weighted-plurality-vote classifier}, via the argmax operation: $h_Q^\text{am}$.
We have the following relationship\footnote{All proofs of results in this paper can be found in the appendix.} between the $\ell$ loss of $Q$ and the output of $h_Q^\text{am}$:
\begin{lemma}
\label{lem:loss-transfer}
    Let $(x, y) \in \R^d \times [k]$, $Q \in \Delta(H)$. Then $\forall y' \ne y: \ell(Q; x,y,y') < 0 \implies  h_Q^\text{am}(x) = y.$
\end{lemma}
In other words, if, on an example $(x, y)$, the expected loss of an $h$ drawn according to $Q$ is negative with respect to every incorrect label $y'$, then the weighted-plurality-vote classifier predicts perfectly on this $(x,y)$ example. 

The following approach to multiclass boosting was developed in \citet{schapire1999improved} with a more general framework in \citet{mukherjee2013theory}. Given $S$, we start with the class of distributions $\Delta(\Snot)$, where $\Snot := \{(x_i, y_i'): i=1, \ldots, m, y_i' \in [k] \setminus \{ y_i \}\}$, the set of \emph{incorrect} example/label pairs. Define the error of $h$ with respect to $D \in \Delta(\Snot)$ as $\err(h, D) := \E_{(x_i, y_i') \sim D}[\ell(h; x_i, y_i, y_i')]$.
We can now define the concept of a weak learner:
\begin{definition}[$\gamma$-Weak Learning]
\label{def:weak-learner}
For some $\gamma > 0$, a $\gamma$-weak learning procedure for the training set $S$ is an algorithm $\alg$ which, when provided any $D \in \Delta(\Snot)$, returns a predictor in $\alg(D) \in H$, such that $\forall D \in \Delta(\Snot): \; \err(\alg(D),D) \leq - \gamma.$
\end{definition}
The following lemma, which follows by an easy min-max analysis using Lemma~\ref{lem:loss-transfer}, shows that the existence of a weak learner implies the existence of a distribution $Q \in \Delta(H)$ such that $h_Q^\text{am}$ is a perfect classifier on the training data:
\begin{lemma}
\label{lem:boosting}
Suppose that for some $\gamma > 0$ we have a $\gamma$-weak learning procedure $\alg$ for the training set $S$. Then there is a distribution $Q \in \Delta(H)$ such that for all $(x, y) \in S$, we have $h_Q^\text{am}(x) = y$.
\end{lemma}
While the lemma is non-constructive as stated, it is easy to make it constructive by using standard regret minimization algorithms \citep{freund1996game}.

\subsection{Adversarial Robustness in Multiclass Boosting}
\label{sec:adversarial-boosting}

The goal of this paper is to produce multiclass classification algorithms that are robust to adversarial perturbations of the input data. To that end, we must adjust our framework to account for the potential that our test data are corrupted. Indeed, we shall imagine that for a given {\em unilabel} hypothesis $h$, when the input $x$ is presented an adversary is allowed to select a small perturbation $z \in B_p(\delta)$, for some $p \geq 1$ and $\delta > 0$ which controls the size of the perturbation, with the goal of modifying the output in order that $h(x+z) \ne y$. 

Inspired by the multiclass boosting framework above, we define
\begin{align}\label{eq:robustloss}
    \elld(h; x, y, y') &= \indic{\exists z \in B_p(\delta): h(x+z) = y'} - \indic{\forall z \in B_p(\delta): h(x+z) = y}.
\end{align}
Let us pause to discuss why this definition is appropriate for the task of adversarial robustness in the multiclass setting. Our vanilla notion of loss measured the extent to which the hypothesis $h$ would output $y'$ instead of $y$ on a given input $x$. The robust loss \eqref{eq:robustloss} measures something more complex: is it \emph{possible} that $h$ could be fooled into predicting something other than $y$, and in particular whether a perturbation \emph{could} produce $y'$.

With this loss function, we can extend the theory of multiclass boosting to the adversarial setting as follows. In analogy with $\err(h, D)$, define the robust error of $h$ with respect to $D \in \Delta(\Snot)$ as 
\[
\err_\delta(h, D) := \E_{(x_i, y_i') \sim D}[\ell_\delta(h; x_i, y_i, y_i')].
\]
We can now define a robust weak learner in analogy with Definition~\ref{def:weak-learner}:
\begin{definition}[$(\gamma,\delta)$-Robust Weak Learning Procedure]
\label{def:robust-weak-learner}
For some $\gamma > 0$, a $(\gamma,\delta)$-robust weak learning procedure for the training set $S$ is an algorithm $\alg$ which, when provided inputs sampled from any $D \in \Delta(\Snot)$, returns a predictor in $\alg(D) \in H$, such that $\forall D \in \Delta(\Snot): \err_\delta(\alg(D),D) \leq - \gamma$.
\end{definition}
With this definition, we can now show the following robust boosting theorem:
\begin{theorem}
\label{thm:robust-boosting}
Suppose that for some fixed $\gamma > 0$ we have a $(\gamma, \delta)$-robust weak learning procedure $\alg$ for the training set $S$. Then there exists a distribution $Q \in \Delta(H)$ such that for all $(x, y) \in S$, and for all $z \in B_p(\delta)$, we have $h_Q^\text{am}(x + z) = y$.
\end{theorem}
The proof of this theorem follows via a reduction to the multiclass boosting technique of Section~\ref{sec:multiclass-boosting}. The reduction relies on the following construction: for any $h \in H$, we define $\tilde h$ to be the {\em multilabel} predictor, defined only on $\{x_1, x_2, \ldots, x_m\}$, as follows: for any $x_i$ and $y \in [k]$, 
\[
    \tilde h(x_i)_y = \begin{cases} \indic{\forall z \in B_p(\delta): h(x_i+z) = y_i} & \text{ if } y = y_i \\
                     \indic{\exists z \in B_p(\delta): h(x_i+z) = y} & \text{ if } y \neq y_i.
    \end{cases}
\]

Let $\tilde{H}$ be the class of all $\tilde h$ predictors constructed in this manner. The mapping $h \mapsto \tilde h$ converts a $(\gamma,\delta)$-robust weak learning algorithm for $H$ into a $\gamma$-weak learning algorithm for $\tilde{H}$, which implies, via Lemma~\ref{lem:boosting}, that there is a distribution $\tilde{Q} \in \Delta(\tilde{H})$ such that for all $(x, y) \in S$, we have $\tilde{h}_{\tilde{Q}}(x) = y$. Now our defintion of the class $\tilde{H}$ implies that if $Q \in \Delta(H)$ is the distribution obtained from $\tilde{Q}$ by applying the reverse map $\tilde h \mapsto h$, then $h_Q^\text{am}$ achieves perfect robust accuracy on $S$.

Just as for Lemma~\ref{lem:boosting}, Theorem~\ref{thm:robust-boosting} can also be made constructive via standard regret minimization algorithms. The result is a boosting algorithm that operates in the following manner. Over a series of rounds, the booster generates distributions in $\Delta(\Snot)$, and then the weak learning algorithm $\alg$ is invoked to find a good hypothesis for each distribution. These hypotheses are then combined (generally with non-uniform weights, as in AdaBoost) to produce the final predictor. Thoerem~\ref{thm:robust-boosting} parallels classical boosting analyses of algorithms like AdaBoost which show how boosting reduces training error. As for generalization, similar to classical analyses bounds can be obtained either by controlling the capacity of the boosted classifier by the number of boosting stages times the capacity of the base function class, or via a margin analysis. We do not include them since even for standard (non-adversarial) training of deep networks, existing generalization bounds are overly loose. In the section~\ref{sec:experiments} we will empirically demonstrate that our boosting based approach indeed leads to state-of-the-art generalization performance.

\subsection{Robust boosting via one-vs-all weak learning}
\label{sec:ova-boosting}

\newcommand\ellova{\ell^{\text{ova}}_\delta}
\newcommand\errova{\err^{\text{ova}}_\delta}

Unfortunately, in practical scenarios when the number of classes $k$ is very large it becomes difficult to implement a weak boosting procedure via the weak learning algorithm of Definition~\ref{def:robust-weak-learner}. This is primarily because the support of the distributions generated by the booster is $\Snot$, which is of size $m(k-1)$, and thus evaluating the robust error rate $\err_\delta(h, D)$ for any hypothesis trypically requires us to find a perturbation $z$ for every $(x_i, y')$ pair with $y' \neq y_i$ such that $h(x_i + z) = y'$. This search for $k-1$ possible perturbations on every example may be prohibitively expensive.

We now present a version of weak learning which is easier to check practically: the $k-1$ searches for a perturbation per example are 
reduced to 1 search per example. We start by loosening the definition of the robust loss \eqref{eq:robustloss} to the following one-vs-all loss $\ellova: H \times \R^d \times [k] \rightarrow -1, 1]$ as
\begin{align}
\ellova(h; x, y) &= \indic{\exists z \in B_p(\delta): h(x+z) \neq y} \notag  - \indic{\forall z \in B_p(\delta): h(x+z) = y}  \notag\\
&= 2\indic{\exists z \in B_p(\delta): h(x+z) \neq y} - 1. \label{eq:ova-loss}
\end{align}
It is clear from the definition that for any $h \in H$, $x \in \mathbb{R}^d$, and $y, y' \in [k]$, we have 
\begin{equation}
\label{eq:loss-domination}
  \ell_\delta(h; x, y, y') \leq \ellova(h; x, y).  
\end{equation}
Now, given a distribution $D \in \Delta(S)$, we define the robust one-vs-all error rate of $h$ on $D$ as $\errova(h, D) := \E_{(x, y) \sim D}[\ellova(h; x, y)]$.
We can now define a robust one-vs-all weak learner:
\begin{definition}[$(\gamma,\delta)$-Robust One-vs-All Weak Learning Procedure]
\label{def:robust-ova-weak-learner}
For some $\gamma > 0$, a $(\gamma,\delta)$-robust weak one-vs-all learning procedure for the training set $S$ is an algorithm $\alg$ which, when provided $D \in \Delta(S)$, returns a predictor in $\alg(D) \in H$, such that $\forall D \in \Delta(S): \; \errova(\alg(D),D) \leq - \gamma.$
\end{definition}
The inequality \eqref{eq:loss-domination} implies that the requirement on $\alg$ in Definition~\ref{def:robust-ova-weak-learner} is more stringent than the one in Definition~\ref{def:robust-weak-learner}, which implies that a robust one-vs-all weak learner is sufficient for boostability, via Theorem~\ref{thm:robust-boosting}. Notice that the support of the distribution $D$ in Definition~\ref{def:robust-ova-weak-learner} is $S$ rather than $\Snot$. Thus, constructing a robust one-vs-all weak learner amounts to just {\em one} search for a perturbation per example: we just need to find a perturbation that makes the hypothesis output {\em any} incorrect label.

\subsection{Robust boosting for score-based predictors}
\newcommand\dunif{D_\text{unif}}
\newcommand\ellsce{\ell^{\text{ce}}_\delta}
\newcommand\ellce{\ell^{\text{ce}}}

We now derive a robust boosting algorithm when the base class of predictors is {\em score-based} via analogy with the unilabel predictor case. If a class $H$ of unilabel predictors admits a $(\gamma, \delta)$-robust one-vs-all weak learner for a given training set $S$, then by Theorem~\ref{thm:robust-boosting}, there exists a distribution $Q^* \in \Delta(H)$ such that $h_{Q^*}^\text{am}$ has perfect robust accuracy on $S$. Equivalently, since $\ellova(h_{Q^*}^\text{am}; x, y) = -1$ for any $(x, y) \in S$, $Q^*$ is a solution to the following ERM problem\footnote{Indeed, {\em any} solution $Q'$ to the ERM problem~\eqref{eq:Q-erm} yields a classifier $h_{Q'}^\text{am}$ with perfect robust accuracy on $S$.}:
\begin{equation}
\label{eq:Q-erm}
  \inf_{Q \in \Delta(H)} \frac{1}{m}\sum_{i=1}^m\ellova(h_Q^\text{am}; x_i, y_i).
\end{equation}
However this ERM problem is intractable in general, so instead we can hope to find $Q^*$ via a surrogate loss for $\ellova$. For the purpose of our experiments, we choose
to use a robust version of the softmax cross-entropy surrogate loss defined as
\begin{equation}
\label{eq:ellsce}
    \ellsce(f; x, y) := \sup_{z \in B_p(\delta)} \ellce(f(x+z), y)  
\end{equation}
where $\ellce$ is the standard cross entropy loss. Lemma~\ref{lem:surrogate} in Appendix~\ref{app:proofs} shows that this is a valid surrogate loss.
Thus, using Lemma~\ref{lem:surrogate} for $f = h_Q$, we may attempt to solve \eqref{eq:Q-erm} by solving the potentially more tractable problem: $\inf_{Q \in \Delta(H)} \frac{1}{m}\sum_{i=1}^m \ellsce(h_Q; x_i, y_i)$.  

We can now extend the above 
to a class $F$ of score-based predictors as follows. First, since score-based predictors have outputs in $\R^k$ rather than $\{0, 1\}^k$, we combine them via linear combinations rather than convex combinations, as in the case of unilabel predictors. Thus, let $\text{span}(F)$ be the set of all possible finite linear combinations of functions in $F$. We index these functions by a weight vector $\beta$ supported on $F$ which has only finitely many non-zero entries, so that $f_\beta$ is the corresponding linear combination. Then, the goal of boosting is to solve the following ERM problem:
\begin{equation}
\label{eq:surrogate-beta-erm}
  \inf_{f_\beta \in \text{span}(F)} \frac{1}{m}\sum_{i=1}^m \ellsce(f_\beta; x_i, y_i).
\end{equation}
By Lemma~\ref{lem:surrogate}, any near-optimal $\beta$ for this problem yields a unilabel predictor $f^\text{am}_{\beta}$ which can be expected to have low robust error rate. Furthermore, note that for any $x \in \R^d$, the mapping $\beta \mapsto f_\beta(x)$ is linear in $\beta$ and $\ellce$ is a convex function in its first argument. Since the $\sup$ operation preserves convexity, it follows that problem~\eqref{eq:surrogate-beta-erm} is {\em convex} in $\beta$. This raises the possibility of solving~\eqref{eq:surrogate-beta-erm} using a gradient-based procedure, which we describe next.

\subsection{Greedy stagewise robust boosting}
\label{sec:greedy-stagewise}
\newcommand\lossfn{\mathcal{L}}

We solve the problem \eqref{eq:surrogate-beta-erm} using the standard boosting paradigm of greedy stagewise fitting \citep{friedman2001greedy,alr}. Assume that the function class $F$ is parametric with a parameter space $\mathcal{W} \subseteq \R^p$, and for every $w \in \mathcal{W}$ we denote by $f_w$ the function in $F$ parameterized by $w$. 
Greedy stagewise fitting builds a solution to the problem \eqref{eq:surrogate-beta-erm} in $T$ stages, for some $T \in \mathbb{N}$. Define $f^{(0)}$ to be the identically $0$ predictor. In each stage $t$, for $t = 1, 2, \ldots, T$, the method constructs the predictor $f^{(t)}$ via the following greedy procedure:
\begin{gather}
  \{\beta_t, w_t\} = \arg\min_{\beta \in \R, w \in \mathcal{W}} \frac{1}{m}\sum_{i=1}^m \ellsce(f^{(t-1)} + \beta f_{w}; x_i, y_i) \label{eq:greedy-step}\\
  f^{(t)} = f^{(t-1)} + \beta_t f_{w_t}. \label{eq:update-step}
\end{gather}
Notice that evaluating $\ellsce(f^{(t-1)} + \beta f_{w}; x_i, y_i)$ involves an inner optimization problem, due to the $\sup_{z_i \in B_p(\delta)}$ for each $i \in [m]$. For each step of gradient descent on $\beta, w$, we solve this inner optimization problem via gradient ascent on the $z_i$'s with projections on $B_p(\delta)$ at each step (this is basically the PGD method of \citet{madry2017towards}). 


Since the inner optimization problem 
involves computing $f^{(t-1)}(x_i + z_i)$ for potentially many values of $z_i$, it becomes necessary to store the parameters of the previously learned functions in memory at each stage. 
In order to derive a practical implementation, we consider an approximate version 
where we replace each $f^{(t-1)}(x_i + z_i)$ evaluation by the value $f^{(t-1)}(x_i)$. Thus the only information we store for the previously learned functions are the values $f^{(t-1)}(x_i)$ for $i \in [m]$, and instead 
we solve: 
\begin{equation}
\label{eq:adv-greedy}
\min_{\beta \in \R, w \in \mathcal{W}} \frac{1}{m}\sum_{i=1}^m \sup_{z_i \in B_p(\delta)} \ellce(f^{(t-1)}(x_i) + \beta f_{w}(x_i+z_i), y_i).
\end{equation}
Empirically, this leads to a significant savings in memory and run time, and as we will show later, our approximate procedure manages to achieve state of the art results. Finally, our boosting framework can be naturally extended to settings where one only has an approximate way of computing the robust error of a hypothesis. Furthermore, we also show that we can also apply our framework to combine weak learners that output a certified radius guarantee along with its prediction, to construct boosted predictors with the same radius guarantee. See Appendix~\ref{sec:extensions} for details.

\section{Experiments}
\label{sec:experiments}
We experiment with three standard image datasets, namely, MNIST \citep{lecun1998gradient}, CIFAR-10 \citep{krizhevsky2009learning} and ImageNet \citep{russakovsky2015imagenet}. We refer the reader to Appendix~\ref{sec:app-experiments} for a discussion of hyperparameter choices.
As described in Section~\ref{sec:greedy-stagewise}, the training algorithm is a greedy stagewise procedure where at each stage we approximately solve \eqref{eq:adv-greedy} by gradient descent on $\beta, w$ and the PGD method of \citet{madry2017towards} for the inner optimization involving the $z_i$'s. 

A crucial component in our practical implementation is the use of {\em cyclic learning rate schedules} in the outer gradient descent (operating on $\beta, w$). When optimizing \eqref{eq:adv-greedy} at stage $t$, we initialize the parameters of the new predictor with the parameters of the previously learned model, i.e., $f^{(t-1)}$. Furthermore, we start the training from a high learning rate and then decrease it to zero following a specific schedule. The use of cyclic learning rates has recently been demonstrated as an effective way to create ensembles of neural networks \citep{loshchilov2016sgdr, huang2017snapshot, smith2017cyclical, izmailov2018averaging}. 
Starting training from a large learning rate has the effect of forcing the training process in the $t$th stage to explore a different part of the parameter space, thereby leading to more diverse and effective ensembles. Following the approach of \citet{loshchilov2016sgdr} and \citet{huang2017snapshot} we use cyclic learning rates as follows. 
In stage $t$ of boosting, we train $\beta_t, w_t$ using SGD over $N_t = 2^{t-1}N_1$ epochs of the training data, where $N_1 \in \mathbb{N}$ is a hyperparameter. The learning rate used in stochastic gradient descent after processing each minibatch is set to $\tfrac 1 2 \eta_{\text{max}}(1+ \cos(\alpha_{\text{cur}}\pi))$, where $\eta_{\text{max}}$ is a tunable hyperparameter similar to the learning rate in standard epoch-wise training schedules, 
and $\alpha_{\text{cur}} \in [0, 1]$ is the fraction of the $N_t$ epochs in the current stage $t$ that have been completed so far.
Thus, over the course of $N_t$ epochs the learning rate goes from $\eta_{\text{max}}$ to $0$ via a cosine schedule. 
 In all our experiments we set $\eta_{\text{min}}$ to be zero and set $\eta_{\text{max}}$ to be a tuneable hyperparameter, similar to the learning rate in standard epoch-wise training schedules. 
 
This leads to our {\em greedy stagewise adversarial boosting} algorithm (see Appendix~\ref{sec:app-experiments}). Each boosting stage, consists of a gradient descent loop to minimize the robust loss, with each iteration of gradient descent implemented by a PGD loop to find a perturbation with high loss.
\begin{figure*}[h]
  \hspace*{\fill}%
  \includegraphics[width=6cm]{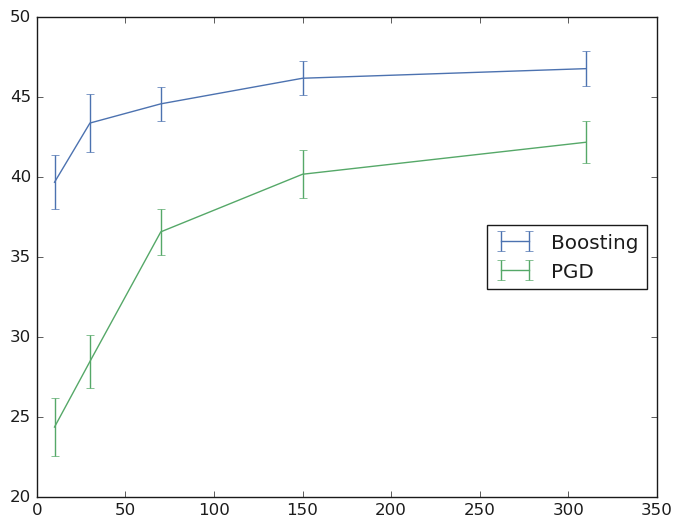}%
  \includegraphics[width=6cm]{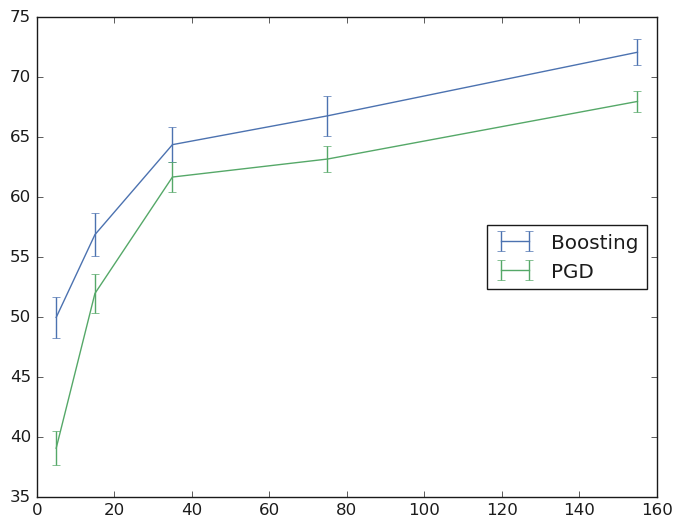}%
  \hspace*{\fill}%
  \caption{\label{fig:scatter}{\em Left:} Comparison of accuracy for a fixed training budget on CIFAR-10 for $\ell_\infty$ robustness. {\em Right:} Comparison of accuracy for a fixed training budget on CIFAR-10 for $\ell_2$ robustness. The x-axis represents the number of epochs and the y-axis represents robust test accuracy.}
\end{figure*}

%

\begin{table*}[t]
\centering
\begin{tabular}{ |c|c|c|c|c|c|c|c|c|c| }
\hline
$\epsilon$ & $0.25$ & $0.5$ & $0.75$ & $1$ & $1.25$ & $1.5$ & $1.75$ & $2$ & $2.25$\\ 
\hline
Bubeck et al. & $73$ & $58$ & $48$ & $38$ & $33$ & $29$ & $24$ & $18$ & $16$\\
\hline
ResNet-20 & 72.6 & 55.34 & 46.44 & 33 & 31 & 27.6 & 22.1 & 16.1 & 12.75\\
\hline
ResNet-32 & {\bf 73.71} & 57.4 & {\bf 50.1} & {\bf 39.07} & 32.7 & {\bf 29.3} & {\bf 25.2} & 17.8 & {\bf 16.8}\\
\hline
\end{tabular}
\begin{tabular}{ |c|c|c|c|c|c|c|c| }
\hline
$\epsilon$ & $0.5$ & $1$ & $1.5$ & $2$ & $2.5$ & $3$ & $3.5$\\ 
\hline
Bubeck et al. & $56$ & $45$ & $38$ & $28$ & $26$ & $20$ & $17$\\
\hline
ResNet-20 & 53.5 & 44 & 32.7 & 24.2 & 24.7 & 17.5 & 16.1\\
\hline
ResNet-32 & {\bf 57.1} & 45.4 & {\bf 39.2} & 27.8 & 26.1 & {\bf 22} & {\bf 18.6}\\
\hline
\end{tabular}
\caption{\label{tbl:l2}
Comparison of certified $\ell_2$ accuracies on CIFAR-10 (top) and ImageNet (bottom).}
\end{table*}

\paragraph{Results on $\ell_\infty$ Robustness.}
In this section we compare robustness to $\ell_{\infty}$ perturbations for classifiers trained via the PGD based algorithm of \citet{madry2017towards} and via our adversarial boosting algorithm. 
We solve the inner maximization problem in \eqref{eq:adv-greedy} via $7$ steps of the projected gradient descent~(PGD) updates. The PGD is started from a random perturbation of the input point, and uses step-size $1.3 \frac{\delta}{7}$, thereby ensuring that PGD updates can reach the boundary of the $\ell_\infty$ ball. Each model is tested by running 20 steps of PGD (also started from a random perturbation of the input point, with step-size $1.3 \frac{\delta}{20}$) with 10 random restarts. For the MNIST dataset, we train a ResNet-110 architecture \citep{he2016deep} to train a single model via adversarial training and compare it to a boosted ensemble of $5$ models using ResNet-12 or ResNet-18 as base predictors. When training the boosted ensemble we set the maximum learning rate $\eta_{\text{max}} = 0.01$ set $N_1 = 5$. For the CIFAR-10 dataset we train a single model on ResNet-110 architecture and compare it to an ensemble of $5$ base predictors using the ResNet-20 and ResNet-32 architectures. Here we set $N_1 = 10$ and $\eta_{\text{max}} = 0.01$. Finally, for the ImageNet dataset we train a single ResNet-50 architecture. Our base predictors in this setting are either the ResNet-20 or the ResNet-32 architectures. We use the same setting of $N_1$ and $\eta_{\text{max}}$ as in CIFAR-10. Our results are shown in Table~\ref{tbl:linf}
As can be seen, in each case our ensemble based training achieves better robust test accuracy as compared to PGD based training of a single model. Furthermore, as Figure~\ref{fig:train-time} shows that our algorithm is able to achieve these accuracy gains at a significantly lower training cost, as a result of working with smaller base predictors. 

\begin{table}[h]\small
\centering
\begin{tabular}{ |c|c|c|c||c|c|c|c|c|c|c|c|c| }
\hline
& \multicolumn{3}{|c||}{MNIST} & & \multicolumn{4}{|c|}{CIFAR-10} & \multicolumn{4}{|c|}{ImageNet}\\
\hline
$\epsilon$ & $0.03$ & $0.07$ & $0.3$ &
$\epsilon$ & $2$ & $4$ & $8$ & $16$& $2$ & $4$ & $8$ & $16$\\ 
\hline
PGD & $99.4$ & $96.7$ & $91.4$ &  & $69$ & $63.4$ & $46$ & $28.9$& $37$ & $35.4$ & $11$ & $4$\\
\hline
ResNet12 & 99.6 & 96.8 & 93.4 & ResNet20 & 69.2 & 62.7 & 45.6 & 27.4& 33.8 & 31.2 & {\bf 12.2} & {\bf 6}\\
\hline
ResNet18 & {\bf 99.6} & {\bf 97} & {\bf 93.8} & ResNet32 & {\bf 70.1} & {\bf 64.6} & {\bf 47.4} & {\bf 29.7}& {\bf 39.7} & {\bf 38.1} & {\bf 14} & {\bf 7.1}\\
\hline
\end{tabular}
\caption{\label{tbl:linf}
Comparison of robust $\ell_\infty$ accuracies on MNIST, CIFAR-10, and ImageNet.}
\vskip -2em
\end{table}

%

\paragraph{Results on $\ell_2$ Robustness.}

To demonstrate the generality of our approach we also apply our boosting based framework to design classifiers that are certifiably robust to $\ell_2$ perturbations. Following recent work of \citet{cohen2019certified} and \citet{salman2019provably}, for a score-based predictor $f$, let $\tilde{f}$ be the smoothed version of $f$ defined as $\tilde{f}(x) = \E_{\epsilon \sim \mathcal{N}(0, \sigma^2 I)} [f(x +\epsilon)]$, where $\sigma$ is a hyperparameter. We apply the greedy stagewise algorithm of Section~\ref{sec:greedy-stagewise} to the class $\tilde{F} = \{\tilde{f}:\ f \in F\}$ instead of $F$. The inner optimization in \eqref{eq:adv-greedy} is again done using $4$ steps of PGD started from the input point with step-size $\frac{\delta}{8}$.
After boosting, we obtain a score-based predictor $f$, which we transform into another score-based predictor $g$ as in \citep{salman2019provably}: $
g(x)_y = \Pr_{\epsilon \sim \mathcal{N}(0, \sigma^2 I)} \big(\arg\max_{y'} f(x+\epsilon)_{y'} = \{y\} \big). 
$
It was shown in the work of \citet{cohen2019certified} that the resulting unilabel predictor $g^\text{am}$ has a radius guarantee: at any point $x$, the prediction of $g^\text{am}$ does not change up to an $\ell_2$ radius of $\frac{\sigma}{2} \big(\Phi^{-1}(g(x)_y) - \Phi^{-1}(g(x)_{y'}) \big)$,
where $y$ and $y'$ are the classes corresponding the maximum and second maximum entries of $g(x)$ respectively, and $\Phi$ is the Gaussian cdf function. In practice, $\Phi^{-1}(g(x)_y)$ and $\Phi^{-1}(g(x)_{y'})$ can be estimated to high accuracy via Monte Carlo sampling. 

We run our adversarial boosting procedure on CIFAR-10 and ImageNet datasets and at the end we compare the certified accuracy of our classifier as compared to the single model trained via the approach of \citet{salman2019provably}. 
We use the same setting of the perturbation radius and the corresponding setting of $\sigma$ as in \citep{salman2019provably}. In each case, we approximate the smoothed classifier by sampling $2$ noise vectors. 
For both CIFAR-10 and ImageNet we train an ensemble of $5$ base predictors using either the ResNet-20 or the ResNet-32 architectures. In contrast, the work of \citet{salman2019provably} uses Reset-110 for training on the CIFAR-10 dataset and a ResNet-50 architecture for training on the ImageNet dataset. In both cases, we set the hyperparameters $\eta_{\text{max}} = 0.3$ and $N_1 = 10$. Our results on certified accuracy as shown in Table~\ref{tbl:l2}.
In each case we either outperform or match the state of the art guarantee achieved by the work of \citet{salman2019provably}. 

Finally, Figure~\ref{fig:scatter} shows test accuracy achieved by our method as compared to the PGD approach (for $\ell_\infty$ robustness) and adversarial smoothing (for $\ell_2$ robustness) at various intervals (wall clock time) during the training process. In each case our method achieves significantly higher test accuracies. We acknowledge that recent and concurrent works have investigated fast algorithms for adversarial training~\citep{wong2020fast, shafahi2019adversarial}. Our boosting based framework complements these approaches and in principle these methods can be used as base predictors in our general framework. We consider this as a direction for future work.

\section{Conclusion}
We demonstrated that provably robust adversarial booting is indeed possible and leads to an efficient practical implementation. The choice of the architectures used in
the base predictors play a crucial role and it would be interesting to develop principled approaches to search for the ``right'' base predictors. When training a large number of weak predictors, inference time will lead to additional memory and computational overheads. These are amplified in the context of adversarial robustness where multiple steps of PGD need to be performed per example, to verify robust accuracy. It would be interesting to explore approaches to reduce this overhead. A promising direction is to explore the use of distillation \citep{hinton2015distilling} to replace the final trained ensemble with a smaller one. 

\bibliography{paper}
\bibliographystyle{abbrvnat}

\newpage
\newpage

\appendix


\section{Proofs}
\label{app:proofs}
We first restate and prove Lemma~\ref{lem:loss-transfer} and Lemma~\ref{lem:boosting}.\\
{\bf  Lemma~\ref{lem:loss-transfer}.} {\em
    Let $(x, y) \in \R^d \times [k]$, and $Q \in \Delta(H)$. Then
    \[
    \forall y' \ne y: \ell(Q; x,y,y') < 0 \implies  h_Q^\text{am}(x) = y.
    \]
}
\begin{proof}
    Observe
    \begin{align*}
        &  \forall y' \ne y: \E_{h \sim Q}\left[\indic{y' \in h(x)} - \indic{y \in h(x)}\right] < 0\\
        \implies & \forall y' \ne y:  \mathop{\text{Pr}}_{h \sim Q} \Big( y \in h(x) \Big ) > \mathop{\text{Pr}}_{h \sim Q} \Big( y' \in h(x) \Big ) \\
        \implies &  \{y\} = \mathop{\arg\max}_{y'' \in [k]}\E_{h \sim Q}[h_{y''}(x)] =: h_Q^\text{am}(x) \\
        \implies &  h_Q^\text{am}(x) = y.
    \end{align*}
    Note: the strict inequality implies unique $\arg\max$.
\end{proof}
{\bf Lemma~\ref{lem:boosting}.} {\em 
Suppose that for some fixed $\gamma > 0$ we have a $\gamma$-weak learning procedure $\alg$ for the training set $S$. Then there exists a distribution $Q \in \Delta(H)$ such that for all $(x, y) \in S$, we have $h_Q^\text{am}(x) = y$.
}
 \begin{proof}
 This proof is a straightforward adaptation of the result of  \citet{freund1996game}. From the weak learning guarantee we have that
 \begin{align*}
              & \forall D \in \Delta(\Snot),  \exists h \in H:\err(h, D) \leq - \gamma \\
     \implies & \forall D \in \Delta(\Snot): \min_{h \in H} \err(h, D) \leq - \gamma \\
     \implies & \max_{D \in \Delta(\Snot)} \min_{h \in H} \err(h, D) \leq - \gamma. 
 \end{align*}
 By the Minimax Theorem we get that
 \begin{align*}
     &  \min_{Q \in \Delta(H)} \max_{D \in \Delta(\Snot)} \E_{h \sim Q}[\err(h, D)] \leq - \gamma \\
     \implies &  \min_{Q \in \Delta(H)} \max_{(x_i, y'_i) \in \Snot} \ell(Q; x_i, y_i, y'_i) \leq - \gamma \\
     \implies &  \exists Q \in \Delta(H), \forall (x_i, y'_i) \in \Snot: \ell(Q; x_i, y_i, y'_i) \leq \ - \gamma.
 \end{align*}
 Hence from Lemma~\ref{lem:loss-transfer} we get that for all $(x, y) \in S$, $h_Q^\text{am}(x) = y$.
\end{proof}
Next we restate and prove our main theorem (Theorem~\ref{thm:robust-boosting}) regarding boosting for adversarial robustness. \\
{\bf Theorem~\ref{thm:robust-boosting}.} {\em 
Suppose that for some fixed $\gamma > 0$ we have a $(\gamma, \delta)$-robust weak learning procedure $\alg$ for the training set $S$. Then there exists a distribution $Q \in \Delta(H)$ such that for all $(x, y) \in S$, and for all $z \in B_p(\delta)$, we have $h_Q^\text{am}(x + z) = y$.
}
\begin{proof}
The proof is via a reduction to the standard multilclass boosting technique of Lemma~\ref{lem:boosting}.
The reduction uses the following construction: for any $h \in H$, we define $\tilde h$ to be the {\em multilabel} predictor, defined only on $\{x_1, x_2, \ldots, x_m\}$, as follows: for any $x_i$ and $y \in [k]$, 
\[
    \tilde h(x_i)_y = \begin{cases} \indic{\forall z \in B_p(\delta): h(x_i+z) = y_i} & \text{ if } y = y_i \\
                     \indic{\exists z \in B_p(\delta): h(x_i+z) = y} & \text{ if } y \neq y_i.
    \end{cases}
\]

Let $\tilde{H}$ be the class of all $\tilde h$ predictors constructed in this manner. Notice that for any $h \in H$, $(x, y) \in S$ and $y' \neq y$ we have that
$$
\ell_\delta(h; x, y, y') = \ell(\tilde{h}; x, y, y').
$$
Hence the mapping $h \mapsto \tilde h$ converts a $(\gamma,\delta)$-robust weak learning algorithm for $H$ into a $\gamma$-weak learning algorithm for $\tilde{H}$, which implies, via Lemma~\ref{lem:boosting}, that there is a distribution $\tilde{Q} \in \Delta(\tilde{H})$ such that for all $(x, y) \in S$, we have $\tilde{h}^\text{am}_{\tilde{Q}}(x) = y$. This implies that for any $y' \neq y$, we have
\begin{align*}
&\Pr_{h \sim Q}[\forall z \in B_p(\delta): h(x+z) = y] \\
&> \Pr_{h \sim Q}[\exists z \in B_p(\delta): h(x+z) = y'],	
\end{align*}
which implies for any given $z \in B_p(\delta)$, we have
\begin{equation} \label{eq:q-argmax}
\Pr_{h \sim Q}[h(x+z) = y] > \Pr_{h \sim Q}[h(x+z) = y'].	
\end{equation}
Let $Q \in \Delta(H)$ be the distribution obtained from $\tilde{Q}$ by applying the reverse map $\tilde h \mapsto h$. Then \eqref{eq:q-argmax} implies that for any $(x, y) \in S$, $y' \neq y$, and $z \in B_p(\delta)$, we have
\[h_Q(x+z)_y > h_Q(x+z)_{y'}.\] Thus, $h_Q^\text{am}(x+z) = y$, as required.
\end{proof}
Let $\ellce: \R^k \times [k] \to \R$ be the standard softmax cross-entropy loss, defined as $\ellce(v, y) := -\ln\left(\tfrac{\exp(v_y)}{\sum_{y'} \exp(v_{y'})}\right)$. Then, the robust softmax cross-entropy loss is defined as follows: for any $f: \R^d \to \R^k$, and $(x, y) \in \R^d \times [k]$, define
The following lemma shows that this yields a valid surrogate loss for $\ellova$, up to scaling and translation:
\begin{lemma}
\label{lem:surrogate}
Let $f: \R^d \to \R^k$ be any score-based predictor. Then for any $(x, y) \in \R^d \times [k]$, we have 
\[\ellova(f^\text{am}; x, y) \leq \tfrac{2}{\ln(2)} \ellsce(f; x, y) - 1.\]
\end{lemma}
\begin{proof}
Recall that 
$$
\ellce(v, y) := -\ln\left(\tfrac{\exp(v_y)}{\sum_{y'} \exp(v_{y'})}\right).
$$
It is easy to see that $\ellce(v,y) \geq 0$. Furthermore, if there exists $y' \neq y$ such that $v_{y'} \geq v_y$, then $\ellce(v,y) \geq \ln 2$. This implies that $\ellce(v, y) \geq \ln(2) \indic{y \not\in \arg\max_{y'}\{v_{y'}\}}$. Hence we have that for any score based predictor $f$ and example $(x,y)$,
\begin{align*}
\ellova(f^\text{am}; x, y) &= 2\indic{\exists z \in B_p(\delta): f^{\text{am}}(x+z) \neq y} - 1\\
&= 2\cdot \sup_{z \in B_p(\delta)}\indic{f^{\text{am}}(x+z) \neq y} - 1 \\
&\leq \tfrac{2}{\ln(2)} \sup_{z \in B_p(\delta)} \ellce(f(x+z); x, y) - 1\\
&= \tfrac{2}{\ln(2)} \ellsce(f; x, y) - 1.
\end{align*}
\end{proof}

Next, we prove Lemma~\ref{lem:apx-weak-learner} regarding boosting using an approximate checker:\\
{\bf Lemma~\ref{lem:apx-weak-learner}.}
{\em Suppose that we have access to a $(c, \delta)$-approximate checker $\mathcal{A}$ for $H$. Then given any $D \in \Delta(S)$, we can use $\mathcal{A}$ to either find an $h \in H$ such that $\err^\text{ova}_{\delta/c}(h, D) \leq -\gamma$, or certify that for all $h \in H$, we have $\errova(h, D) > -\gamma$.}
\begin{proof}
Let $h \in H$ be any unilabel predictor. We use the shorthand ``$\mathcal{A}(h; x, y) = z$'' to denote the event that the output $\mathcal{A}(h; x, y)$ is a point $z \in B_p(\delta)$ such that $h(x+z) \neq y$. For any $(x, y) \in S$, we have the following inequality from the definition of a $(c, \delta)$-approximate checker:
\begin{align*}
	&\indic{\exists z \in B_p(\delta):\ h(x+z) \neq y} \\
	&\geq \indic{\mathcal{A}(h; x, y) = z} \\
	&\geq \indic{\exists z \in B_p(\delta/c):\ h(x+z) \neq y}.
\end{align*}
This implies that
\begin{align*}
	\errova(h, D) &\geq \E_{(x, y) \sim D}[2\indic{\mathcal{A}(h; x, y) = z} - 1]\\
	 &\geq \err^\text{ova}_{\delta/c}(h, D).
\end{align*}
Thus, if for some $h \in H$ we have
\[\E_{(x, y) \sim D}[2\indic{\mathcal{A}(h; x, y) = z} - 1] \geq -\gamma,\] 
then we have $\err^\text{ova}_{\delta/c}(h, D) \leq -\gamma$ as well. Otherwise, if for all $h \in H$ we have 
\[\E_{(x, y) \sim D}[2\indic{\mathcal{A}(h; x, y) = z} - 1] > -\gamma,\] 
then $\errova(h, D) > -\gamma$ as well.
\end{proof}

Finally, we turn to proving Theorem~\ref{thm:certified} regarding boosting with certified accuracy. We need the following lemma which is a stronger form of Lemma~\ref{lem:boosting} applying to one-vs-all weak learners:
\begin{lemma}
\label{lem:ova-boosting}
Suppose that for some fixed $\gamma > 0$ we have a $(\gamma, 0)$-robust one-vs-all weak learning procedure $\alg$ for the training set $S$. Then there exists a distribution $Q \in \Delta(H)$ such that for all $(x, y) \in S$, we have $\Pr_{h \sim Q}[h(x) = y] \geq \frac{1+\gamma}{2}$.
\end{lemma}
 \begin{proof}
 This proof is a straightforward adaptation of the result of  \citet{freund1996game}. From the weak learning guarantee we have that
 \begin{align*}
              & \forall D \in \Delta(S),  \exists h \in H:\err^\text{ova}_0(h, D) \leq - \gamma \\
     \implies & \forall D \in \Delta(S): \min_{h \in H} \err^\text{ova}_0(h, D) \leq - \gamma \\
     \implies & \max_{D \in \Delta(S)} \min_{h \in H} \err^\text{ova}_0(h, D) \leq - \gamma. 
 \end{align*}
 By the Minimax Theorem we get that
 \begin{align*}
     &  \min_{Q \in \Delta(H)} \max_{D \in \Delta(S)} \E_{h \sim Q}[\err^\text{ova}_0(h, D)] \leq - \gamma \\
     \implies &  \min_{Q \in \Delta(H)} \max_{(x, y) \in S} \E_{h \sim Q}\left[2\indic{h(x) \neq y} - 1\right] \leq - \gamma \\
     \implies &  \exists Q \in \Delta(H), \forall (x, y) \in S: \Pr_{h \sim Q}[h(x) = y] \geq \tfrac{1+\gamma}{2}.
 \end{align*}
 \end{proof}
Using Lemma~\ref{lem:ova-boosting} reduction in the proof of Theorem~\ref{thm:robust-boosting}, we get the following stronger boosting statement for robust one-vs-all weak learners:
\begin{theorem} 
\label{thm:ova-boosting}
Suppose that for some fixed $\gamma > 0$ we have a $(\gamma, \delta)$-robust one-vs-all weak learning procedure $\alg$ for the training set $S$. Then there exists a distribution $Q \in \Delta(H)$ such that for all $(x, y) \in S$, and for all $z \in B_p(\delta)$, we have $\Pr_{h \sim Q}[h(x + z) = y] \geq \frac{1+\gamma}{2}$.
\end{theorem}
We can now prove Theorem~\ref{thm:certified}:\\
{\bf Theorem~\ref{thm:certified}.} {\em 
Let $H$ be a class of predictors with radius guarantees. Let $\delta^*$ be the maximal radius such that for any distribution $D \in \Delta(S)$, there exists an $h \in H$ with $\text{acc}_{\delta^*}(h, D) \geq \frac{1}{2} + \frac{\gamma}{2}$, for some constant $\gamma > 0$. Then, there exists a distribution $Q \in \Delta(H)$ and function $\rho_Q: \R^d \to \R$ such that the predictor $f = (h_Q^\text{am}, \rho_Q)$ has perfect certified accuracy at radius $\delta^*$ on $S$.
}
\begin{proof}
Note that for any distribution $D \in \Delta(S)$ and $h \in H$, if $\text{acc}_{\delta^*}(h, D) \geq \frac{1}{2} + \frac{\gamma}{2}$, then $\err^\text{ova}_{\delta^*}(h, D) \leq -\gamma$. Thus we can apply Theorem~\ref{thm:robust-boosting} to conclude that there exists a distribution $Q \in \Delta(H)$ such that for all $(x, y) \in S$ and for all $z \in B_p(\delta^*)$, we have $\Pr_{h \sim Q}[h(x+z) = y] \geq \frac{1+\gamma}{2}$. Evidently the predictor $h_Q$ has perfect certified accuracy at radius $\delta^*$ on $S$. To define $\rho_Q$, we assume without loss of generality that $Q$ has finite support (indeed, the constructive form of Theorem~\ref{thm:robust-boosting} yields this) $\{h_1, h_2, \ldots, h_N\}$, and let $Q(h_i)$ be the $Q$-mass of $h_i$. 

Now given any input $x \in \R^d$, let $h_Q^\text{am}(x) = y$. Let $h_{(1)}, h_{(2)}, \ldots, h_{(N)}$ be an ordering of the hypotheses in the support of $Q$ arranged in increasing order of their radius guarantees for $x$: i.e. if $r_i$ is the radius guarantee of $h_{(i)}(x)$, then $r_1 \leq r_2 \leq \cdots \leq r_N$. For any index $i \in \{1, 2, \ldots, N\}$, let $h_Q^{(i)} = \sum_{j=i}^N Q(h_{(j)}) h_{(j)}$. By a linear scan in the set $\{r_1, r_2, \ldots, r_N\}$, find the largest index $i^*$ such that the following condition holds:
\[ \sum_{i < i^*} Q(h_{(i)}) + \max_{y' \neq y} h_Q^{(i^*)}(x)_{y'} \leq h_Q^{(i^*)}(x)_y, \]
Note that such an index exists because the index $1$ satisfies this condition. It is easy to see that the prediction of $h_Q^\text{am}$ is unchanged upto radius $r_{i*}$, so we can set $\rho_Q(x) = r_{i^*}$.
\end{proof}

\section{Extensions of Boosting Framework}
\label{sec:extensions}

We now return to setting of Section~\ref{sec:ova-boosting}, working with a unilabel class of predictors $H$. We describe two extensions of our boosting analysis in the following subsections.

\subsection{Boosting with Approximate Checkers}

In order to construct a robust one-vs-all weak learning algorithm (see Definition~\ref{def:robust-ova-weak-learner}) to employ in a boosting algorithm, at the very least we need a way to compute the robust one-vs-all error rate $\errova(h, D)$ of a given predictor $h$. This requires us to find adversarial perturbations for a given input $(x, y)$, a task which is NP-hard in general. To counter this problem, \citet{AwasthiDV19} introduced the notion of an {\em approximate checker} and showed that such a checker can be implemented in polynomial time for polynomial threshold functions and a subclass of 2-layer neural networks. An approximate checker is defined as follows:
\begin{definition}[$(c, \delta)$-approximate checker]
\label{def:apx-checker}
For some constant $c \geq 1$, a $(c, \delta)$-approximate checker for hypothesis class $H$ is an algorithm $\mathcal{A}$ with that has the following specification: for any unilabel predictor $h \in H$, and any example $(x, y) \in \R^d \times [k]$, the output $\mathcal{A}(h; x, y)$ is
\[\begin{cases}
	z\!\! \in\!\! B_p(\delta)\!\!: h(x\!\! +\!\! z) \neq y & \text{if }\ \exists z'\!\! \in\!\! B_p(\tfrac{\delta}{c})\!\!: h(x\!\! +\!\! z') \neq y; \\
	\textnormal{\texttt{Good}} & \text{if }\ \forall z\!\! \in\!\! B_p(\delta)\!\!: h(x\!\! +\!\! z) = y;\\
	\emptyset & \text{otherwise.}
\end{cases}
\]
We use $\emptyset$ to mean that the output can be arbitrary.
\end{definition}
Using an approximate checker, we can construct a robust one-vs-all weak learner as described in the following lemma:
\begin{lemma}
\label{lem:apx-weak-learner}
 Suppose that we have access to a $(c, \delta)$-approximate checker $\mathcal{A}$ for $H$. Then given any $D \in \Delta(S)$, we can use $\mathcal{A}$ to either find an $h \in H$ such that $\err^\text{ova}_{\delta/c}(h, D) \leq -\gamma$, or certify that for all $h \in H$, we have $\errova(h, D) > -\gamma$.
\end{lemma}
Now suppose that $H$ is rich enough so that for any $D \in \Delta(S)$, there exists an $h \in H$ such that $\errova(h, D) \leq -\gamma$. Then Theorem~\ref{thm:robust-boosting} implies that the weak learner described in Lemma~\ref{lem:apx-weak-learner} can be used to find a predictor with perfect robust accuracy, up to a radius of $\delta/c$, on the training set $S$.

\subsection{Certified accuracy}

\citet{cohen2019certified}, following the work of \citet{lecuyer2019certified}, constructed predictors which have an appealing radius guarantee property, obtained via randomized smoothing of the outputs. Specifically, for each input $x$, the predictor outputs not only a class $y$ for $x$, but also a radius $\delta$ such that the predicted class is guaranteed to be $y$ for any perturbation of $x$ within a ball of $\ell_2$ radius $\delta$ around $x$. Thus, the predictor $h$ is a pair of functions $(h_L, h_R)$ with $h_L: \R^d \to [k]$ and $h_R: \R^d \to \R$, so that on input $x$, the predicted label is $h_L(x)$ and is guaranteed to not change within an $\ell_2$ ball of radius $h_R(x)$ around $x$.
\begin{definition}[Certified accuracy]
\label{def:certified-accuracy}
Given a distribution $D$ over examples and a radius $\delta \geq 0$, the {\em certified accuracy} of a predictor $h: \R^d \times [k] \times \R$ at radius $\delta$ is defined as
\[ \text{acc}_\delta(h, D) := \Pr_{(x, y) \sim D}[h_L(x) = y \text{ and } h_R(x) \geq \delta].\]
\end{definition}
In our boosting framework, if the weak learner outputs predictors with radius guarantees as above, we would like the boosted predictor to also have radius guarantees. The following theorem shows that it is indeed possible to construct such a boosted predictor whenever the class of predictors satisfies a certain weak-learning type assumption based on certified accuracy:
\begin{theorem}
\label{thm:certified}
Let $H$ be a class of predictors with radius guarantees. Let $\delta^*$ be the maximal radius such that for any distribution $D \in \Delta(S)$, there exists an $h \in H$ with $\text{acc}_{\delta^*}(h, D) \geq \frac{1}{2} + \frac{\gamma}{2}$, for some constant $\gamma > 0$. Then, there exists a distribution $Q \in \Delta(H)$ and function $\rho_Q: \R^d \to \R$ such that the predictor $f = (h_Q^\text{am}, \rho_Q)$ has perfect certified accuracy at radius $\delta^*$ on $S$.
\end{theorem}

\section{Experiments and Algorithm Pseudo Code}
\label{sec:app-experiments} 
Below we present the pseudo code for our boosting base training procedure. Recall the cyclic learning rate schedule where the learning rate used in stochastic gradient descent after processing each minibatch is set to
\begin{align}
\label{eq:cyclic-lr}
\tfrac 1 2 \eta_{\text{max}}(1+ \cos(\alpha_{\text{cur}}\pi)),
\end{align}

\begin{algorithm}[H]
\caption{Greedy Stagewise Adversarial Boosting with Offset Approximation}
\begin{algorithmic}[1]
\item[] \textbf{Input:} $S = \{(x_1, y_1), \dots, (x_m, y_m)\}$, $N_1$, $\eta_{\text{max}}, T, \epsilon$.
\item [] \textbf{Output:} Boosted classifier: $f = \sum_{t=1}^T \beta_t f_{w_t}$.
\FOR{$t=1$ to $T$}
\STATE Initialize the parameters $\beta_t, w_t$ to be zero.
\STATE $N_t = 2^{t-1} N_1$.
\FOR{epochs in $1$ to $N_t$}
\FOR{$j$ in $1$ to $\text{num\_batches}$}
\STATE Get next mini batch and update learning rate $\eta$ using \eqref{eq:cyclic-lr}.
\STATE Update $\{w_t, \beta_t \}$ via optimizing \eqref{eq:adv-greedy} on the current mini batch.
\ENDFOR
\ENDFOR
\STATE Update $f^t \leftarrow f^{t-1} + \beta_t f_{w_t}$. 
\ENDFOR
\STATE Output $f = \sum_{t=1}^T \beta_t f_{w_t}$.
\end{algorithmic}
\label{alg:boost}
\end{algorithm}

\subsection{Discussion of Training Settings}
All the hyperparameters in our experiments were chosen via cross validation and we discuss the specific choice of learning rates in the relevant sections. For the adversarial smoothing based method of~\citet{salman2019provably} we used the same settings of hyperparameters as described by the authors in~\citet{salman2019provably}. The results on MNIST and CIFAR-10 are averaged over $10$ runs, while the results on ImageNet are averaged over $3$ runs. Figure~\ref{fig:scatter} shows the comparison of robust accuracy achieved by our ensemble and that achieved by training a single model trained via PGD \citep{madry2017towards} or smooth adversarial training \citep{salman2019provably}, for a fixed training budget as measured in wall clock time. As can be seen our ensemble always achieves higher robust accuracies for the same training budget. 

Although we can train our ensembles much faster, the size of the final ensembles that we produce are typically larger than the model used for training via existing approaches \citep{madry2017towards, salman2019provably}. In the case of CIFAR-10 our ensemble of 5 ResNet-32 models is a bit larger than the ResNet-110 architecture that we compare to. In the case of ImageNet our ensemble of 5 ResNet-20 or ResNet-32 architectures are significantly larger than the ResNet-50 architecture that we compare to. In general, the choice of the base architecture plays a crucial role in the success of the boosting based approach. Our experiments reveal that the base architecture has to be of a certain complexity for the ensemble to be able to compete with the training of a single model. Furthermore, at inference time the entire ensemble has to be loaded into memory. For our purposes this causes negligible overhead as our ensembles do not contain too many base predictors. In general however, this memory overhead at inference should be taken into consideration. Furthermore, at  inference time, our ensembles allow for a natural parallelism that can be exploited via using multiple GPUs. We did not explore this in the current work.

\end{document}